\documentclass[3p, authoryear]{elsarticle}

\usepackage[utf8]{inputenc}
\usepackage[T1]{fontenc}
\usepackage[english]{babel}
\usepackage{xspace}
\usepackage{amsmath,amssymb,amsthm}

\usepackage{lineno}
\modulolinenumbers[5]

\makeatletter
\def\ps@pprintTitle{%
     \let\@oddhead\@empty
     \let\@evenhead\@empty
     \def\@oddfoot{\footnotesize\itshape
       \@title\hfill\today}
     \let\@evenfoot\@oddfoot}
\makeatother
   
\usepackage{hyperref}
\usepackage{tgpagella}
\usepackage{booktabs}

\usepackage{tikz}
\usetikzlibrary{calc}
\usepackage{gnuplot-lua-tikz}
\definecolor{mycyan}{rgb}{0,0.2,0.4}
\definecolor{mycyan2}{rgb}{0.08,0.58,0.89}
\definecolor{verylightgray}{rgb}{0.9,0.9,0.9}

\newcommand{\soscs}{sequences of sincere choices\xspace}
\newcommand*{\al}[1]{
\begin{tikzpicture}
[baseline=(letter.base)]
\node[draw,circle,inner sep=1pt] (letter) {$\mathrm #1$};
\end{tikzpicture}
}

\newcommand{\fixme}[1]{{\color{red}[FIXME] #1}}

\newcommand{\vect}[1]{\overrightarrow{#1}}
\newcommand{\Def}{\buildrel\hbox{\tiny \textit{def}}\over =}
\newcommand{\pol}{\textsf{P}\xspace}
\newcommand{\np}{\textsf{NP}\xspace}
\newcommand{\conp}{\textsf{coNP}\xspace}

\newtheorem{definition}{Definition}
\newtheorem{prop}{Proposition}
\newtheorem{example}{Example}
\newtheorem{observation}{Observation}
\newtheorem{corollary}{Corollary}

\DeclareMathOperator{\argmax}{argmax}
\DeclareMathOperator{\argmin}{argmin}
\DeclareMathOperator{\best}{best}

\usepackage[english,vlined,ruled,linesnumbered]{algorithm2e}
\SetKw{From}{from}
\SetKw{To}{to}
\SetKw{Choose}{choose}
\SetKw{NonSeq}{NonSeq}
\SetKwInput{Input}{Input}
\SetKwInput{Output}{Output}

\usepackage{stmaryrd}

\newcommand{\llb}{\llbracket}                               
\newcommand{\rrb}{\rrbracket}                               

\makeatletter
\newcommand{\hidden}[1]{\ifhmode \@bsphack \@esphack \fi}
\makeatother

\hypersetup{%
  colorlinks=true,
  citecolor=red,
  linkcolor=red,
  anchorcolor=red,
  urlcolor=red,
  pdfstartview=FitH,
  pdfdisplaydoctitle=true,
  pdftoolbar=true,
  pdfmenubar=true
}

\makeatletter
\hypersetup{%
  pdftitle={Efficiency, Sequenceability and Deal-Optimality in Fair Division of Indivisible Goods},
  pdfauthor={Aurélie Beynier, Sylvain Bouveret, Michel Lemaître, Nicolas Maudet, Simon Rey}
}
\makeatother

\begin{document}

\renewcommand{\labelitemi}{$\bullet$}

\hypersetup{%
  citecolor=mycyan,
  linkcolor=mycyan,
  anchorcolor=mycyan,
  urlcolor=mycyan,
}

\begin{frontmatter}

\title{%
  Efficiency, Sequenceability and Deal-Optimality in Fair Division of Indivisible Goods\footnote{This article is an extension of a paper presented at COMSOC \citep{BL-COMSOC16}.}
}

\author[lip6]{Aurélie Beynier}
\author[lig]{Sylvain Bouveret}
\author[onera]{Michel Lemaître}
\author[lip6]{Nicolas Maudet}
\author[lip6,ens]{Simon Rey}

\address[lip6]{LIP6, Sorbonne Université}
\address[lig]{LIG, Univ. Grenoble-Alpes}
\address[onera]{Formerly ONERA Toulouse}
\address[ens]{ENS Paris Saclay}


\begin{abstract}
In fair division of indivisible goods, using sequences of sincere
  choices (or picking sequences) is a natural way to allocate the
  objects. The idea is as follows: at each stage, a designated
  agent picks one object among those that remain. 
  Another intuitive way to obtain an allocation is to give objects to agents in the first place, and to let agents exchange them as long as such ``deals'' are beneficial. 
This paper investigates these notions, when agents have additive preferences over objects, and unveils surprising connections between them, and with other efficiency and fairness notions. 
 In particular, we show that an allocation is sequenceable iff it is optimal for a certain type of deals, namely cycle deals involving a single object.  Furthermore, any Pareto-optimal allocation is sequenceable, but not the converse. 
Regarding fairness, we show that 
  an allocation can be envy-free and non-sequenceable, but that every
  competitive equilibrium with equal incomes is sequenceable. 
 To complete the picture, we show how some domain restrictions may affect the relations between these notions. 
   Finally, we experimentally explore the links between the scales of efficiency and fairness.
\end{abstract}


\end{frontmatter}


\section{Introduction}

In this paper, we investigate fair division of indivisible goods. In
this problem, a set of indivisible objects or goods has to be
allocated to a set of agents, taking into
account the agents' preferences about the
objects. This classical collective decision making problem has plenty
of practical applications, among which the allocation of space
resources \citep{Lemaitre99,Bianchessi07}, of tasks to workers in
crowdsourcing market systems, papers to reviewers
\citep{goldsmith2007ai} or courses to students \citep{Budish11}.

This problem can be tackled from two different perspectives. The first
possibility is to resort to a benevolent entity in charge of
collecting in a \emph{centralized} way the preferences of all the
agents. This entity then computes an allocation that
takes into account these preferences and satisfies some fairness
(\textit{e.g.} envy-freeness) and efficiency (\textit{e.g.}
Pareto-optimality) criteria, or optimizes a well-chosen social welfare
ordering.  The second possibility is to have a \emph{distributed}
point of view, \textit{e.g.} by starting from an initial allocation
and letting the agents negotiate to swap their objects
\citep{Sandholm98,Chevaleyre05ijcai}. 
A somewhat intermediate approach consists in allocating the objects to
the agents using a \emph{protocol}, which allows to build an allocation interactively by asking the agents a sequence of questions. Protocols are at the heart of works mainly
concerning the allocation of divisible resources (cake-cutting) \citep{Brams96}, but
have also been studied in the context of indivisible goods
\citep{Brams96,brams2012undercut}. 

In this paper, we focus on a particular allocation protocol:
\emph{sequences of sincere choices} (also known as \emph{picking
  sequences}). This very simple protocol works as follows.
A central authority chooses a sequence of agents before the protocol
starts, having as many agents as the number of objects (some agents
may appear several times in the sequence). Then, each agent appearing
in the sequence is asked to choose in turn one object among those that
remain. For instance, according to the sequence
$\langle 1,2,2,1 \rangle$, agent $1$ is going to choose first, then
agent $2$ will pick two consecutive objects, and agent $1$ will take
the last object. This protocol, actually used in a lot of
everyday situations, 
has been studied for the first time by \cite{KC71}. Later,
\cite{Brams00} have studied a particular version of this protocol,
namely alternating sequences, in which the sequence of agents is
restricted to a balanced ($\langle 1,2,2,1...\rangle$) or strict
($\langle 1,2,1,2...\rangle$) alternation of agents. \cite{BL-IJCAI11}
have further formalized this protocol, whose properties (especially
related to game theoretic aspects) have been characterized
\citep{KalinowskiNW13,KalinowskiNWX13}. Finally,
\cite{Aziz2015possible} have studied the complexity of problems
related to finding whether a particular assignment (or bundle) is
achievable by a particular class of picking sequences.  In their work
(not specifically dedicated to picking sequences) focusing on a
situation where the agents have ordinal preferences,
\cite{BramsKing05} make an interesting link between this protocol and
Pareto-optimality, showing, among others, that picking sequences
always result in a Pareto-optimal allocation, but also that every
Pareto-optimal allocation can be obtained in this way.

In this paper, we elaborate on these ideas and analyze the links
between sequences, certain types of deals among agents, and some
efficiency and fairness properties, in a more general model in which
the agents have numerical additive preferences on the objects. Our
main contributions are the following.  We give a formalization of the
link between allocations and
\soscs, 
highlighting a simple characterization of the sequenceability of an
allocation (Section \ref{sec:seqAlloc}). Then, we show that in this
slightly more general framework than the one by Brams \emph{et al.},
surprisingly, Pareto-optimality and sequenceability are not equivalent
anymore (Section \ref{sec:Pareto}). By unveiling the connection
between sequenceability and cycle deals among agents (Section
\ref{sec:deals}), we obtain a rich ``scale of efficiency'' that allows
us to characterize the degree of efficiency of a given
allocation. Interestingly, some domain restrictions have significant
effects on this hierarchy (Section \ref{sec:restricted-domains}).  We
also highlight a link between sequenceability and another important
economical concept: the competitive equilibrium from equal income
(CEEI). Another contribution is the experimental exploration of the
links between the scale of efficiency and fairness properties.


\section{Model and Definitions}
\label{sec:modele}

The aim of the fair division of indivisible goods, also called
MultiAgent Resource Allocation (MARA), is to allocate a finite set of
\emph{objects} $\mathcal{O} = \{1,\dots,m\}$ to a finite set of
\emph{agents} $\mathcal{N} = \{1,\dots,n\}$. A \emph{sub-allocation}
on $\mathcal{O}' \subseteq \mathcal{O}$ is a vector
${\vect{\pi}}^{|\mathcal{O}'} = \langle
\pi_1^{|\mathcal{O}'},\dots,\pi_{n}^{|\mathcal{O}'}\rangle$ of
\emph{bundles} of objects, such that $\forall i, \forall j$ with
$i \neq j : \pi_i^{|\mathcal{O}'}\cap\pi_j^{|\mathcal{O}'}=\emptyset$
(a given object cannot be allocated to more than one agent) and
$\bigcup_{i\in\mathcal{N}} \pi_i^{|\mathcal{O}'} = \mathcal{O}'$ (all
the objects from $\mathcal{O}'$ are allocated).
$\pi_i^{|\mathcal{O}'} \subseteq \mathcal{O}'$ is called agent $i$'s
\emph{share} on $\mathcal{O}'$.  ${\vect{\pi}}^{|\mathcal{O}''}$ is a
sub-allocation of ${\vect{\pi}}^{|\mathcal{O}'}$ when
$\pi^{|\mathcal{O}''}_i \subseteq \pi^{|\mathcal{O}'}_i$ for each
agent $i$. Any sub-allocation ${\vect{\pi}}^{|\mathcal{O}} $ on the
entire set of objects will be denoted ${\vect{\pi}} $ and just called
\emph{allocation}.

Any satisfactory allocation must take into account the agents'
preferences on the objects. Here, we will make the classical
assumption that these preferences are \emph{numerically additive}.
Each agent $i$ has a \emph{utility function} $u_i:2^\mathcal{O} \to \mathbb{R}^+$
measuring her satisfaction $u_i(\pi)$ when she obtains share $\pi$,
which is defined as follows:
\begin{equation*} \label{eq:utilite} u_i(\pi) \Def
  \sum_{k\in\pi}w(i,k),
\end{equation*}
where $w(i, k)$ is the weight given by agent $i$ to object
$k$. This assumption, as restrictive as it may seem, is made by a
lot of authors \cite[for instance]{Lipton04,Bansal06} and is
considered as a good compromise between expressivity and conciseness.

\hidden{If we put things together:}
\begin{definition}
  An instance of the \emph{additive multiagent resource allocation
    problem} (add-MARA instance for short) $I = \langle \mathcal{N}, \mathcal{O}, w \rangle$ is a tuple
  \hidden{$(\mathcal{N}, \mathcal{O}, w)$,}  with $\mathcal{N}$ and $\mathcal{O}$ as defined above and \hidden{ where:
  \begin{itemize}
  \item $\mathcal{N} = \{1, \dots ,i, \dots, n\}$ is a set of
    $n$ \emph{agents};
  \item $\mathcal{O} = \{1, \dots ,k, \dots m\}$ is a set of
    $m$ \emph{objects},
  \item} $w: \mathcal{N} \times \mathcal{O} \to \mathbb{R}^+$ is a mapping with $w(i, k)$ being the weight given by agent $i$ to
    object $k$.
  \hidden{\end{itemize}}

We say that the agents' preferences \emph{are strict on objects} if,
for each agent $i$ and each pair of objects $k \neq l$, we have
$w(i, k) \neq w(i, l)$. Similarly, we say that the agents'
preferences are \emph{strict on shares} if, for each agent $i$ and
each pair of shares $\pi\neq\pi'$, we have $u_i(\pi) \neq
u_i(\pi')$. Note that preferences strict on shares entails preferences strict on
objects; the converse is false.

  We will denote by $\mathcal{P}(I)$ the set of allocations for $I$.
\end{definition}

\hidden{
\begin{observation} \label{rem:str} Strict order on shares entails strict
  order on objects; the converse is false.
\end{observation}
}

The following definition will play a prominent role.

\begin{definition}
  Given an agent $i$ and a set of objects $O$, let
  $\best(O, i) = \argmax_{k\in O} w(i, k)$ be the subset 
  of objects in $O$ having the highest weight for agent $i$ (such
  objects will be called \emph{top} objects of $i$). 
  
  A (sub-)allocation ${\vect{\pi}}^{|\mathcal{O}'}$ is said \emph{frustrating} if
  no agent receives one of her top objects in ${\vect{\pi}}^{|\mathcal{O}'}$
  (formally:
  $\best(\mathcal{O}', i) \cap {\vect{\pi}}^{|\mathcal{O}'}_i = \emptyset$ for
  each agent $i$), and \emph{non-frustrating} otherwise.
\end{definition}

\hidden{It should be emphasized that this notion of frustrating
  allocation was already present but implicit in the work by Brams and
  King \cite{BramsKing05}. Here, we bring this concept out because it
  will lead to a nice characterization of sequenceable allocations, as
  we will see later.  } \medskip

In the following, we will consider a particular way of allocating
objects to agents: allocation by sequences of sincere
choices. Informally the agents are asked in turn, according to a
predefined sequence, to choose and pick a top object among the
remaining ones.

\begin{definition} Let $I = \langle \mathcal{N}, \mathcal{O}, w \rangle$ be an add-MARA instance. A
  \emph{sequence of sincere choices} (or simply \emph{sequence} when
  the context is clear) is a vector of $\mathcal{N}^m$.
We will denote by $\mathcal{S}(I)$ the set of possible sequences for the instance $I$.
\end{definition}

\hidden{Let $\vect{\sigma} \in \mathcal{S}(I)$. $\vect{\sigma}$ is said to \emph{generate} allocation
${\vect{\pi}}$ iff ${\vect{\pi}}$ can be obtained as a possible
result of a non-deterministic 
algorithm on input $I$ and $\vect{\sigma}$ which simply makes agents chose one of their top object at their turn in the sequence.\footnote{The non-determinism comes from the fact that at her turn, an agent might be indifferent between several top objects.}}

Let $\vect{\sigma} \in \mathcal{S}(I)$. $\vect{\sigma}$ is said to
\emph{generate} allocation $\vect{\pi}$ if and only if $\vect{\pi}$
can be obtained as a possible result of the
non-deterministic\footnote{The algorithm contains an instruction
  $\Choose$ splitting the control flow into several branches, building
  all the allocations generated by $\vect{\sigma}$.}
Algorithm~\ref{algo:ss} on input $I$ and $\vect{\sigma}$.

\begin{algorithm}
  \caption{Execution of a sequence}
  \label{algo:ss}
  \Input{an instance $I = \langle \mathcal{N}, \mathcal{O}, w \rangle$ and a sequence $\vect{\sigma} \in \mathcal{S}(I)$}
  \Output{an allocation ${\vect{\pi}} \in \mathcal{P}(I)$}

  ${\vect{\pi}} \gets$ empty allocation (such that $\forall i \in \mathcal{N} : \pi_i = \emptyset$)\;
  $\mathcal{O}_1 \gets \mathcal{O}$\;
  \For{$t$ \From 1 \To $m$}{
    $i \gets \sigma_t$\;
    \Choose object $o_t \in \best(\mathcal{O}_t, i)$\label{line:guess} \;
    $\pi_i \gets \pi_i  \cup \{ o_t \} $ \;
    $\mathcal{O}_{t+1} \gets \mathcal{O}_t \setminus  \{ o_t \} $
  }
\end{algorithm}

\begin{definition} \label{def:sequenceable} An allocation ${\vect{\pi}}$ is
  said to be \emph{sequenceable} if there exists a sequence $\vect{\sigma}$
  that generates ${\vect{\pi}}$, and \emph{non-sequenceable} otherwise. For
  a given instance I, we will denote by $s(I)$ the binary relation
  defined by $(\vect{\sigma},{\vect{\pi}})\in s(I)$ if and only if ${\vect{\pi}}$ can be
  generated by $\vect{\sigma}$.
\end{definition}

\begin{example} \label{ex:A1} Let $I$ be the instance represented by
  the following weight matrix:\footnote{In this example and the
    following ones, we represent instances by a matrix where
    the value at row $i$ and column $k$ represents the weight
    $w(i, k)$. We also use $ab...$ as a shorthand
    for $\{a, b, ...\}$.}
  \[  \left(
    \begin{array}{rrr}
      8 & 2 & 1 \\
      5 & 1 & 5 
    \end{array}
  \right) \]
   The binary relation $s(I)$ between $\mathcal{S}(I)$ and $\mathcal{P}(I)$ can be graphically
   represented as follows:

   
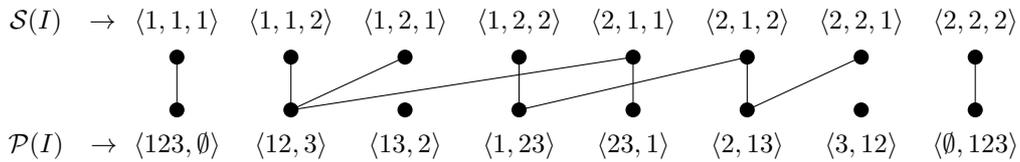
\vspace{-0.2cm}  
   \begin{center}
     \begin{tikzpicture}
       \draw (-1.5, 0) node[circle, inner sep=2pt, outer sep=5pt] (node-S1) {} node[above, outer sep=5pt] {$\mathcal{S}(I)$\quad$\to$};

       \draw (0, 0) node[circle, fill=black, inner sep=2pt, outer sep=5pt] (node-a) {} node[above, outer sep=5pt] {$\langle 1, 1, 1\rangle$};
       \draw (1.5, 0) node[circle, fill=black, inner sep=2pt, outer sep=5pt] (node-b) {} node[above, outer sep=5pt] {$\langle 1, 1, 2\rangle$};
       \draw (3, 0) node[circle, fill=black, inner sep=2pt, outer sep=5pt] (node-c) {} node[above, outer sep=5pt] {$\langle 1, 2, 1\rangle$};
       \draw (4.5, 0) node[circle, fill=black, inner sep=2pt, outer sep=5pt] (node-d) {} node[above, outer sep=5pt] {$\langle 1, 2, 2\rangle $};
       \draw (6, 0) node[circle, fill=black, inner sep=2pt, outer sep=5pt] (node-e) {} node[above, outer sep=5pt] {$\langle 2, 1, 1\rangle$};
       \draw (7.5, 0) node[circle, fill=black, inner sep=2pt, outer sep=5pt] (node-f) {} node[above, outer sep=5pt] {$\langle 2, 1, 2\rangle$};
       \draw (9, 0) node[circle, fill=black, inner sep=2pt, outer sep=5pt] (node-g) {} node[above, outer sep=5pt] {$\langle 2, 2, 1\rangle$};
       \draw (10.5, 0) node[circle, fill=black, inner sep=2pt, outer sep=5pt] (node-h) {} node[above, outer sep=5pt] {$\langle 2, 2, 2\rangle$};

       \draw (-1.5, -0.7) node[circle, inner sep=2pt, outer sep=5pt] (node-P1) {} node[below, outer sep=5pt] {$\mathcal{P}(I)$\quad$\to$};

       \draw (0, -0.7) node[circle, fill=black, inner sep=2pt, outer sep=5pt] (node-A) {} node[below, outer sep=5pt] {$\langle  123, \emptyset \rangle$};
       \draw (1.5, -0.7) node[circle, fill=black, inner sep=2pt, outer sep=5pt] (node-B) {} node[below, outer sep=5pt] {$\langle  12 , 3 \rangle$};
       \draw (3, -0.7) node[circle, fill=black, inner sep=2pt, outer sep=5pt] (node-C) {} node[below, outer sep=5pt] {$\langle  13, 2 \rangle$};
       \draw (4.5, -0.7) node[circle, fill=black, inner sep=2pt, outer sep=5pt] (node-D) {} node[below, outer sep=5pt] {$\langle  1 , 23 \rangle$};
       \draw (6, -0.7) node[circle, fill=black, inner sep=2pt, outer sep=5pt] (node-E) {} node[below, outer sep=5pt] {$\langle  23 , 1 \rangle$};
       \draw (7.5, -0.7) node[circle, fill=black, inner sep=2pt, outer sep=5pt] (node-F) {} node[below, outer sep=5pt] {$\langle 2 ,13 \rangle$};
       \draw (9, -0.7) node[circle, fill=black, inner sep=2pt, outer sep=5pt] (node-G) {} node[below, outer sep=5pt] {$\langle  3 , 12 \rangle$};
       \draw (10.5, -0.7) node[circle, fill=black, inner sep=2pt, outer sep=5pt] (node-H) {} node[below, outer sep=5pt] {$\langle  \emptyset , 123 \rangle$};

       \foreach \x/\y in {a/A, b/B, c/B, d/D, e/B,  e/E, f/D, f/F, g/F, h/H}
      \draw (node-\x.center) -- (node-\y.center);
     \end{tikzpicture}
   \end{center}

  For instance, sequence $\langle 2, 1, 2 \rangle$ generates two
  possible allocations: $\langle 1, 23 \rangle$ and
  $\langle 2, 13 \rangle$, depending on whether agent 2 chooses object
  1 or 3 that she both prefers. Allocation $\langle 12, 3 \rangle$ can
  be generated by three sequences. Allocations $\langle 13, 2 \rangle$
  and $\langle 3, 12 \rangle$ are non-sequenceable.
\end{example}

\hidden{
\begin{observation} \label{rem:card} For any instance $I$,
  $|\mathcal{S}(I)| = |\mathcal{P}(I)| = n^m$.
\end{observation}
}
For any instance $I$,
  $|\mathcal{S}(I)| = |\mathcal{P}(I)| = n^m$.
Note also that the number of objects allocated to an agent by a sequence is the number of times the agent appears in the sequence. 
  \hidden{
  Formally: for all $(\vect{\sigma},{\vect{\pi}})\in s(I)$ and all
  agents $i$,
  $|\pi_i| = \sum_{k\in\mathcal{O}}[\sigma_k = i]$ where
  $[z = t]$ is 1 if the equality is verified, and 0 otherwise.
  }

  \hidden{Once again, this notion of sequenceability is already
    implicitly} The notion of frustrating allocation and
  sequenceability were already implicitly present in the work by
  \cite{BramsKing05}, and sequenceability has been extensively studied
  by \cite{Aziz2015possible} with a focus on sub-classes of sequences
  (\textit{e.g.}  alternating sequences). However, a fundamental
  difference is that in our setting, the preferences might be non
  strict on objects, which entails that the same sequence can yield
  different allocations (in the worst case, an exponential number), as
  Example~\ref{ex:A1} shows.



\section{Sequenceable allocations}
\label{sec:seqAlloc}

We have seen in Example~\ref{ex:A1} that some allocations are
non-sequenceable. We will now formalize this and give a precise
characterization of sequenceable allocations that is, we will try to
identify under which conditions an allocation is achievable by the
execution of a sequence of sincere choices. We first start by noticing
that in every sequenceable allocation, the first agent of the sequence
gets a top object, so every frustrating allocation is
non-sequenceable.  \hidden{ In this section and in the following one,
  we will give a characterization of sequenceable allocations, that
  is, we will try to identify under which conditions an allocation is
  achievable by the execution of a sequence of sincere choices. The
  question has already been extensively studied by
  \cite{Aziz2015possible}, but in a quite different context ---
  namely, ordinal strict preferences on objects --- and with a
  particular focus on sub-classes of sequences (\textit{e.g.}
  alternating sequences). As we will show, the properties are not
  completely similar in our context.

\subsection{Characterization}
 
We have seen in Example~\ref{ex:A1} that some allocations are
non-sequenceable. We will now formalize this and give a precise
characterization of sequenceable allocations. We first start by
noticing that in every sequenceable allocation, the first agent of the
sequence gets a top object, so every frustrating allocation is
  non-sequenceable. which yields the following remark:

\begin{observation} \label{rem:nonSeq} Every frustrating allocation is
  non-sequenceable.
\end{observation}
 }
\hidden{ 
\begin{example} \label{ex:B} In the following instance,
  the circled allocation $\langle 23, 1 \rangle$ is
  non-sequenceable because it is frustrating.

  \[  \left(
    \begin{array}{ccc}
      2 & \al 1  & \al 1\\
      \al 1 & 2 & 2
    \end{array}
  \right) \]
\end{example}
}
It is even possible to find a non-sequenceable allocation that
gives her top object to one agent (as allocation
$\langle 13, 2 \rangle$ in Example~\ref{ex:A1}) or even to all: 

\begin{example} \label{ex:C} Consider this instance: 

  \[  \left(
    \begin{array}{cccc}
      \al 9 & 8 & \tikz[remember picture]\node[inner sep=0pt] (ul1) {};2 & \al 1 \\
      2 & \al 5 & \al 1 & 4\tikz[remember picture]\node[inner sep=0pt] (br1) {}; 
    \end{array}
  \right) \]

  \tikz[remember picture]\draw[overlay,dotted] ($ (ul1) + (-0.2, 0.4) $) rectangle ($ (br1) + (0.2, -0.15) $);

  In the circled allocation
  ${\vect{\pi}}=\langle 14, 23 \rangle$, every agent receives her
  top object. However, after objects 1 and 2 have been allocated (they
  must be allocated first by all sequences generating ${\vect{\pi}}$), the
  dotted sub-allocation remains.  This
  sub-allocation is obviously non-sequenceable because it is
  frustrating. Hence ${\vect{\pi}}$ is not sequenceable either.
\end{example}

This property of containing a frustrating sub-allocation exactly
characterizes the set of non-sequenceable allocations:
 
\begin{prop}\label{prop:sousAlloc}
  Let $I = \langle \mathcal{N}, \mathcal{O}, w \rangle$ be an instance and ${\vect{\pi}}$ be an allocation of
  this instance. The two following statements are equivalent:
  \begin{itemize}
  \item[(A)] ${\vect{\pi}}$ is sequenceable.
  \item[(B)] No sub-allocation of ${\vect{\pi}}$ is frustrating (in every
    sub-allocation, at least one agent receives a top object).
  \end{itemize}
\end{prop}

\begin{proof} (B) $\Rightarrow$ (A). Let us suppose that for all subsets of
  objects $\mathcal{O}' \subseteq \mathcal{O}$ there is at least one agent
  obtaining one of her top objects in ${\vect{\pi}}^{|\mathcal{O}'}$. We will
  show that ${\vect{\pi}}$ is sequenceable. Let $\vect{\sigma}$ be a
  sequence of agents and $\vect{\mathcal{O}} \in (2^\mathcal{O})^m$
  be a sequence of sets of objects jointly defined as follows:
  \begin{itemize}
  \item $\mathcal{O}_1 = \mathcal{O}$ and $\sigma_1$ is an agent that
    receives one of her top objects in
    ${\vect{\pi}}^{|\mathcal{O}_1}$;
  \item $\mathcal{O}_{t+1} = \mathcal{O}_{t} \setminus \{o_{t}\}$,
    where $o_{t} \in \best(\mathcal{O}_{t}, \sigma_{t})$ and
    $\sigma_t$ is an agent that receives one of her top objects in
    ${\vect{\pi}}^{|\mathcal{O}_t}$, for $t \geq 1$.
  \end{itemize}
  
  From the assumption on ${\vect{\pi}}$, we can check that the sequence
  $\vect{\sigma}$ is perfectly defined. Moreover, ${\vect{\pi}}$ is one of the
  allocations generated by $\vect{\sigma}$.
  
  (A) $\Rightarrow$ (B) by contraposition. Let ${\vect{\pi}}$ be an allocation
  containing a frustrating sub-allocation
  ${\vect{\pi}}^{|\mathcal{O}'}$. Suppose that there exists a sequence $\vect{\sigma}$
  generating ${\vect{\pi}}$.  We can notice that in
  a sequence of sincere choices\hidden{ Algorithm~\ref{algo:ss}}, when an object is allocated to an agent,
  all the objects that are strictly better for her have already been
  allocated at a previous step. Let $k\in\mathcal{O}'$, and let $i$ be
  the agent that receives $k$ in ${\vect{\pi}}$. Since
  ${\vect{\pi}}^{|\mathcal{O}'}$ is frustrating, there is another object
  $l \in \mathcal{O}'$ such that $w(i, l) > w(i, k)$. From the
  previous remark, $l$ is necessarily allocated before $k$ in the
  execution of $\vect{\sigma}$. We can deduce, from the same line of
  reasoning on $l$ and agent $j$ that receives it, that there is
  another object $p$ allocated before $l$ in the execution of
  the sequence. The set $\mathcal{O}'$ being finite, using the same
  argument iteratively, we will necessarily find an object which has
  already been encountered before. This leads to a cycle in the
  precedence relation of the objects in the execution of the
  sequence. Contradiction: no sequence can thus generate ${\vect{\pi}}$.
  \end{proof}

  Beyond the fact that it characterizes a sequenceable allocation, the proof of
  Proposition~\ref{prop:sousAlloc} gives
  a practical way of checking if an allocation is sequenceable, and,
  if it is the case, of computing a sequence that generates this
  allocation. 

  \begin{prop}\label{prop:complexity}
    Let $I = \langle \mathcal{N}, \mathcal{O}, w \rangle$ be an instance and ${\vect{\pi}}$ be an allocation
    of this instance. We can decide in time
    $O(n \times m^2)$ if ${\vect{\pi}}$ is sequenceable.
  \end{prop}

  The proof is based on the execution of
  Algorithm~\ref{algo:testSeq}. This algorithm is similar in spirit to
  the one proposed by \cite{BramsKing05} but
  Algorithm~\ref{algo:testSeq} is more general because (i) it can
  involve non-strict preferences on objects, and (ii) it can conclude
  with non-sequenceability.

  \begin{proof} 
    We show that Algorithm~\ref{algo:testSeq} returns a sequence
    $\vect{\sigma}$ generating the input allocation ${\vect{\pi}}$ if and only if
    there is one. Suppose that the algorithm returns a sequence
    $\vect{\sigma}$. Then, by definition of the sequence (in the loop from line
    \ref{lgn:defSeq1} to line \ref{lgn:defSeq2}), at each step $t$,
    $i =\sigma_t$ can choose an object in $\pi_i$, that is one of
    her top objects. Conversely, suppose the algorithm returns
    \textbf{NonSeq}. Then, at a given step $t$,
    $\forall i$, $\best(\mathcal{O}', i) \cap \pi_i = \emptyset$.
    By definition, ${\vect{\pi}}^{|\mathcal{O}'}$ is therefore, at this step, a
    frustrating sub-allocation of ${\vect{\pi}}$. By
    Proposition~\ref{prop:sousAlloc}, ${\vect{\pi}}$ is thus
    non-sequenceable.
    The loop from line \ref{lgn:defSeq1} to line \ref{lgn:defSeq2}
    runs in time $O(n \times m)$, because searching
    for the top objects in the preferences of each agent is in $O(m)$. This loop being executed $m$
    times, the algorithm runs in 
    $O(n \times m^2)$.

    \begin{algorithm}
      \caption{Sequencing an allocation}
      \label{algo:testSeq}
      \Input{$I = \langle \mathcal{N}, \mathcal{O}, w \rangle$ and ${\vect{\pi}} \in \mathcal{P}(I)$}
      \Output{a sequence $\vect{\sigma}$ generating ${\vect{\pi}}$ or \NonSeq}
      
      $(\vect{\sigma}, \mathcal{O}') \gets (\langle\rangle, \mathcal{O})$\;
      
      \For{$t$ \From 1 \To $m$\label{lgn:defSeq1}}{
        \If{$\exists i$ such that $\best(\mathcal{O}', i) \cap \pi_i \neq \emptyset$}{
          Append $i$ to $\vect{\sigma}$\;
          let $k \in \best(\mathcal{O}', i) \cap \pi_i$\;
          $\mathcal{O}' \gets \mathcal{O}' \setminus \{k\}$\;
        }
        \lElse{
          \Return{\NonSeq\label{lgn:defSeq2}}
        }
      }
      \Return $\vect{\sigma}$\;
    \end{algorithm}
\end{proof}



\section{Pareto-optimality}
\label{sec:Pareto}

An allocation is Pareto-optimal if there is no other allocation
dominating it. In our context, allocation ${\vect{\pi}}'$ dominates
allocation ${\vect{\pi}}$ if for all agent $i$,
$u_i(\pi_i') \geq u_i(\pi_i)$ and
$u_j(\pi_j') > u_j(\pi_j)$ for at least one
agent $j$. 
When an allocation is generated from a sequence, in some sense, a weak
form of efficiency is applied to build the allocation: each successive
(picking) choice is ``locally'' optimal. This raises a natural
question: is every sequenceable allocation Pareto-optimal?

This question has already been extensively discussed independently by
\cite{AzizKWX16} and in a previous version of this work 
\citep{BL-COMSOC16}. We complete the discussion here to give
more insights about the implications of the previous results in
our framework.

\citet[Proposition 1]{BramsKing05} prove the equivalence between
sequenceability and Pareto-optimality. However, they have a different
notion of Pareto-optimality, because \hidden{they only have partial
  ordinal information about the agents' preferences. More precisely,
  in Brams and King's model,} the agents' preferences are given as
linear orders over \emph{objects}\hidden{ (\textit{e.g.}
  $1 \succ 2 \succ 3 \succ 4$)}. To be able to compare bundles, these
preferences are lifted on subsets using the \emph{responsive set
  extension} $\succ_{RS}$. \hidden{, which is similar to the one
  defined in the work by \cite{Aziz15fair}, and \cite{BEL-ECAI10}) for
  SCI-nets.}
  This extension
leaves many bundles incomparable \hidden{(\textit{e.g.} $14$ and $23$ if we
consider the order $1 \succ 2 \succ 3 \succ 4$).}
and leads to  
define \emph{possible} and \emph{necessary}
Pareto-optimality. Brams and King's notion is possible Pareto-optimality. 
\cite{Aziz15fair} show that, given a linear
order $\succ$ on objects and two bundles $\pi$ and $\pi'$,
$\pi \succ_{RS} \pi'$ if and only if $u(\pi) > u(\pi')$ for \emph{all}
additive utility functions $u$ compatible with $\succ$ (that is, such
that $u(k) > u(l)$ if and only if $k \succ l$). This characterization
of responsive dominance yields the following reinterpretation of Brams
and King's result:\hidden{

\begin{prop}[\cite{BramsKing05}]
  Let $\langle \succ_1, \dots, \succ_n \rangle$ be the
  profile of agents' ordinal preferences (represented as linear
  orders). } an allocation ${\vect{\pi}}$ is sequenceable if and only if for
  each other allocation ${\vect{\pi}}'$, there is a set
  $u_1, \dots, u_n$ of additive utility functions,
  respectively compatible with $\succ_1, \dots, \succ_n$
  such that $u_i({\pi_i}) > u_i({\pi_i}')$ for at least one
  agent $i$.
\hidden{\end{prop}}

The latter notion of Pareto-optimality is very weak, because (unlike in our context) the
set of additive utility functions is not fixed --- we just have to find one
that works. Under our stronger notion, \hidden{\footnote{See also \citep{Aziz16realloc} for a related discussion
  about Pareto-optimality.} }
the equivalence between sequenceability and Pareto-optimality no longer holds.\footnote{Actually, since it is known \citep{KeijzerEtalADT09,Aziz16realloc} that testing Pareto-optimality with additive preferences in \conp-complete, and that testing sequenceability is in \pol (Proposition~\ref{prop:complexity}), they cannot be equivalent unless \pol = \conp.}

\begin{example} \label{ex:E} Let us consider the following instance:

  \[ \left(
    \begin{array}{rrr}
      5 & 4 & 2 \\
      8 & 2 & 1 
    \end{array}
  \right) \]

  The sequence $\langle 1, 2, 2 \rangle$ generates allocation
  $A = \langle 1, 23 \rangle$ giving utilities
  $\langle 5, 3 \rangle$. $A$ is then sequenceable but it is dominated by
  $B = \langle 23, 1 \rangle$, giving utilities $\langle 6, 8 \rangle$
  (and generated by $\langle 2, 1, 1 \rangle$). Observe that, under
  ordinal linear preferences, $B$ would not dominate $A$, but
  they would be incomparable.
\end{example}

The last example shows that a sequence of sincere choices does not
necessarily generate a Pareto-optimal allocation. What about the
converse? We can see, as a trivial corollary of the reinterpretation
of Brams and King's result in our terminology, that the answer is
positive \emph{if the preferences are strict on shares}. The following
result is more general, because it holds even without this assumption:

\begin{prop}[\citealp{AzizKWX16,BL-COMSOC16}] \label{prop:POseq} Every
  Pareto-optimal allocation is sequenceable.
\end{prop}

Before giving the formal proof, we illustrate it on a concrete example
\cite[Example 5]{Bouveret2015CharacterizingConflicts}.

\begin{example} \label{ex:EFnonPOprefsStrictes} 
	Let us consider the following instance:
	
	\[
	W = \left(
	\begin{array}{ccccc}
	\dag \al{12} & 15 & \tikz[remember picture]\node[inner sep=0pt] (ul2_simon) {}; \dag{11} & \al  7 & 2\\
	2 & 12 & \al 7 & \dag 15 & \dag \al{11} \\
	15 & \dag \al{20} & 9 & 2 & 1 \tikz[remember picture]\node[inner sep=0pt] (br2_simon) {};
	\end{array}  \right)
	\]
	\tikz[remember picture]\draw[overlay,dotted] ($ (ul2_simon) + (-0.2, 0.45) $) rectangle ($ (br2_simon) + (0.35, -0.2) $);
	
	The circled allocation $\vect{\pi}$ is not sequenceable: indeed, every
	sequence that could generate it should start with
	$\langle 3,1,\dots \rangle$, leaving the frustrating sub-allocation
	$\vect{\rho}$ in a dotted box.
	
	Let us now choose an arbitrary agent who does not receive a top
	object in $\vect{\rho}$, for instance agent $a_1 = 1$. Let $o_1 = 3$
	be her top object (of weight 11 in this case). The agent receiving
	$o_1$ in $\vect{\rho}$ is $a_2=2$. This agent prefers object
	$o_2= 4$ (of weight 15), held by $a_1$, already encountered. We have
	built a cycle
	$ a_2 \stackrel{o_1}{\longrightarrow} a_{1}
	\stackrel{o_2}{\longrightarrow} a_2$, in other words
	$ 2 \stackrel{3}{\longrightarrow} 1 \stackrel{4}{\longrightarrow}
	2$, that tells us exactly how to build another sub-allocation
	dominating $\vect{\rho}$. This sub-allocation can be built by
	replacing in $\vect{\rho}$ the attributions
	$ (a_1 \gets o_2 ) (a_{2} \gets o_1)$ by the attributions
	$ (a_1 \gets o_1) (a_{2} \gets o_{2})$. Hence, each agent involved
	in the cycle obtains a strictly better object than the previous
	one. Doing the same substitutions in the initial allocation
	$\vect{\pi}$ yields an allocation $\vect{\pi}'$ that dominates
	$\vect{\pi}$ (marked with $\dag$ in the matrix $W$ above).
\end{example}

Now we will give the formal proof.\footnote{%
  This proof is similar to the one by \citet[Proposition 1,
  necessity]{BramsKing05}. However we give it entirely because it is
  more general and will be reused in Proposition~\ref{prop:CEEI}. 
\hidden{The central idea of trading cycle is classical and is used,
\textit{e.g.}, by \citet[page 79]{Varian1974} and \citet[Lemma 
2.2]{Lipton04} in the context of envy-freeness.}
  }
\begin{proof} As stated in the example, we will now prove the
  contraposition of the proposition: every non-sequen\-ceable
  allocation is dominated. Let $\vect{\pi}$ be a non-sequenceable
  allocation. From Proposition~\ref{prop:sousAlloc}, in a
  non-sequenceable allocation, there is at least one frustratring
  sub-allocation. Let $\vect{\rho}$ be such a sub-allocation (that can
  be $\vect{\pi}$ itself). We will, from $\vect{\rho}$, build another
  sub-allocation dominating it. Let us choose an arbitrary agent $a_1$
  involved in $\vect{\rho}$, receiving an object not among her top
  ones in $\vect{\rho}$. Let $o_1$ be a top object of $a_1$ in
  $\vect{\rho}$, and let $a_2$ ($\neq a_1$) be the unique agent
  receiving it in $\vect{\rho}$. Let $o_2$ be a top object of $a_2$.
  We can notice that $o_2 \neq o_1$ (otherwise $a_2$ would obtain one
  of her top objects and $\vect{\rho}$ would not be frustrating). Let
  $a_3$ be the unique agent receiving $o_2$ in $\vect{\rho}$, and so
  on.  Using this argument iteratively, we form a path starting from
  $a_1$ and alternating agents and objects, in which two successive
  agents and objects are distinct. Since the number of agents and
  objects is finite, we will eventually encounter an agent which has
  been encountered at a previous step of the path. Let $a_i$ be the
  first such agent and $o_k$ be the last object seen before her in the
  sequence ($a_i$ is the unique agent receiving $o_k$). We have built
  a cycle
  $ a_i \stackrel{o_{k}}{\longrightarrow} a_{k}
  \stackrel{o_{k-1}}{\longrightarrow} a_{k-1} \cdots a_{i+1}
  \stackrel{o_{i}}{\longrightarrow} a_i$ in which all the agents and
  objects are distinct, and that has at least two agents and two
  objects. From this cycle, we can modify $\vect{\rho}$ to build a new
  sub-allocation by giving to each agent in the cycle a top object
  instead of another less preferred object, all the agents not
  appearing in the cycle being left unchanged. More formally, the
  following attributions in $\vect{\rho}$ (and hence in $\vect{\pi}$):
  $ (a_i \gets o_k ) (a_{i+1} \gets o_i) \cdots (a_k \gets o_{k-1})$
  are replaced by:
  $ (a_i \gets o_i) (a_{i+1} \gets o_{i+1}) \cdots (a_k \gets o_{k})$
  where $ (a \gets o ) $ means that $o$ is attributed to $a$. The same
  substitutions operated in $\vect{\pi}$ yield an allocation
  $\vect{\pi}'$ that dominates $\vect{\pi}$.
\end{proof}

\hidden{
Before giving the formal proof, we illustrate it on a concrete example
\cite[Example 5]{Bouveret2015CharacterizingConflicts}.

\begin{example} \label{ex:EFnonPOprefsStrictes} 
 Let us consider the following instance:

  \[
  w = \left(
    \begin{array}{ccccc}
      2 & 12 & \tikz[remember picture]\node[inner sep=0pt] (ul2) {};\al 7 & \dag 15 & \dag \al{11} \\
      \dag \al{12} & 15 & \dag{11} & \al  7 & 2\\
      15 & \dag \al{20} & 9 & 2 & 1 \tikz[remember picture]\node[inner sep=0pt] (br2) {};
    \end{array}  \right)
  \]
  \tikz[remember picture]\draw[overlay,dotted] ($ (ul2) + (-0.2, 0.45) $) rectangle ($ (br2) + (0.35, -0.2) $);
 
  The circled allocation ${\vect{\pi}}$ is not sequenceable: indeed, every
  sequence that could generate it should start with
  $\langle 3,2,\dots \rangle$, leaving the frustrating sub-allocation
  ${\vect{\pi}'}$ in a dotted box.

  Let us now choose an arbitrary agent who does not receive a top
  object in ${\vect{\pi}'}$, for instance agent $a_1 = 2$. Let $o_1 = 3$ be
  her top object (of weight 11 in this case). The agent receiving
  $o_1$ in ${\vect{\pi}'}$ is $a_2=1$. This agent prefers object $o_2= 4$
  (of weight 15), held by $a_1$, already encountered. We have built a
  cycle
  $ (a_1 , o_1) \rightarrow (a_{2} , o_{2}) \rightarrow (a_1, o_{1})$,
  in other words $(2 , 3) \rightarrow (1 , 4) \rightarrow (2, 3)$,
  that tells us exactly how to build another sub-allocation dominating
  ${\vect{\pi}'}$. This sub-allocation can be built by replacing in
  ${\vect{\pi}'}$ the attributions $ (a_1 \gets o_2 ) (a_{2} \gets o_1)$ by
  the attributions $ (a_1 \gets o_1) (a_{2} \gets o_{2})$. Hence, each
  agent involved in the cycle obtains a strictly better object than
  the previous one. Doing the same substitutions in the initial
  allocation ${\vect{\pi}}$ yields an allocation ${\vect{\pi}}'$ that dominates
  ${\vect{\pi}}$ (marked with $\dag$ in the matrix $w$ above).
\end{example}

\begin{proof}
  The proof of this proposition directly follows from
  Observations~\ref{obs:si-wi} and \ref{obs:Pareto}, together with
  Proposition~\ref{pro:cycle-sequence}.
\end{proof}
}

\begin{corollary} \label{cor:nonPO} No frustrating allocation can be
  Pareto-optimal (equivalently, in every Pareto-optimal allocation, at
  least one agent receives a top object).
\end{corollary}

Proposition~\ref{prop:POseq} implies that there exists, for a given
instance, three classes of allocations: (1) non-sequenceable
(therefore non Pareto-optimal) allocations, (2) sequenceable but non
Pareto-optimal allocations, and (3) Pareto-optimal (hence
sequenceable) allocations. These three classes define a ``scale of
efficiency'' that can be used to characterize the allocations. 
What is interesting and new here is the intermediate level.
We will see that this scale can be further detailed. 



\section{Cycle deals-optimality}
\label{sec:deals}

Pareto-optimality can be thought as a reallocation of objects among
agents using improving \emph{deals} \citep{Sandholm98}, as we have
seen, to some extent, in the proof of
Proposition~\ref{prop:POseq}. \emph{Trading cycles} or \emph{cycle
  deals} constitute a sub-class of deals,
which is classical and used, \textit{e.g.}, by \citet[page
79]{Varian1974} and \citet[Lemma 2.2]{Lipton04} in the context of
envy-freeness. Trying to link efficiency concepts with various notions
of deals is thus a natural idea.


\begin{definition}
  Let $\langle \mathcal{N}, \mathcal{O}, w \rangle$ be an add-MARA
  instance and $\vect{\pi}$ be an allocation of this instance. A
  $(N, M)$-cycle deal of $\vect{\pi}$ is a sequence of transfers of
  items
  $d = \langle (i_1, \mathcal{X}_1), \dots, (i_N, \mathcal{X}_N)
  \rangle$, where, for each $j$, $i_j \in
  \mathcal{N}$,$\mathcal{X}_j \subseteq \pi_j$, and
  $|\mathcal{X}_j| \leq M$. The allocation $\vect{\pi}[\leftarrow d]$
  resulting from the application of $d$ to $\vect{\pi}$ is defined as
  follows:
  \begin{itemize}
  \item
    $\pi[\leftarrow d]_{i_j} = \pi_{i_j} \setminus \mathcal{X}_j \cup
    \mathcal{X}_{j-1}$ for $j \in \{ 2, \dots, N \}$;
  \item
    $\pi[\leftarrow d]_{i_1} = \pi_{i_1} \setminus \mathcal{X}_1 \cup
    \mathcal{X}_{N}$;
  \item $\pi[\leftarrow d]_{i} = \pi_{i}$ if
    $i \not\in \{i_1, \dots, i_N\}$.
  \end{itemize}
  A cycle deal
  $\langle (i_1, \mathcal{X}_1), \dots, (i_N, \mathcal{X}_N) \rangle$
  will be written
  $ i_1 \stackrel{\mathcal{X}_1}{\longrightarrow} i_2 \dots i_{N-1}
  \stackrel{\mathcal{X}_{N-1}}{\longrightarrow} i_N
  \stackrel{\mathcal{X}_N}{\longrightarrow} i_1$.
\end{definition}

In other word, in a cycle deal (we omit $N$ and $M$ when they are not
necessary to understand the context), each agent gives a subset of at
most $M$ items from her share to the next agent in the sequence and
receives in return a subset from the previous agent. $(N, 1)$-cycle
deals will be denoted by $N$-cycle deals. $2$-cycle deals
will be called \emph{swap}-deals. Among these cycle deals, some are
more interesting: those where each agent improves her utility by
trading objects. More formally, a deal $d$ will be called
\emph{weakly improving} if $u_i(\pi[\leftarrow d]_i) \geq u_i(\pi_i)$
$\forall i \in \mathcal{N}$ with at least one of these inequalities
being strict, and \emph{strictly improving} if all these inequalities are strict.


Intuitively, if it is possible to improve an allocation by applying an
improving cycle deal, then it means that this allocation is
inefficient. Reallocating the items according to the deal will make
everyone better-off. It is thus natural to derive a concept of
efficiency from this notion of cycle-deal.

\begin{definition}
  An allocation is said to be $>$-$(N, M)$-Cycle Optimal (resp.
  $\geq$-$(N, M)$-Cycle Optimal) if it does not admit any strictly
  (resp. weakly) improving $(K, M)$-cycle deal for any $K \leq N$.
\end{definition}


We begin with easy observations. First, $\geq$-cycle optimality
implies $>$-cycle optimality, and these two notions become equivalent
when the preferences are strict on shares. Moreover, restricting the
size of the cycles and the size of the bundles exchange yield less
possible deals and hence lead to weaker optimality notions. 
\hidden{
We have the following relations:
\begin{observation}\label{obs:$>$-wi}
  \[
    \vect{\pi} \text{ is } \text{ $\geq$-$(N, M)$-CO} \Rightarrow \vect{\pi} \text{ is } \text{ $>$-$(N, M)$-CO},
  \]
  \[
    \vect{\pi} \text{ is } \text{ $\geq$-$(N, M)$-CO} \Rightarrow \vect{\pi} \text{ is } \text{ $\geq$-$(l, l')$-CO $\forall l \leq N$ and $l' \leq M$}
  \]
  \[
    \vect{\pi} \text{ is } \text{ $>$-$(N, M)$-CO} \Rightarrow \vect{\pi} \text{ is } \text{ $>$-$(l, l')$-CO $\forall l \leq N$ and $l' \leq M$}
  \]
\end{observation}

These implications are strict, meaning that the reverse implications
do not hold.}
%
%
%
%
Note that for $N' \leq N$ and $M' \leq M$ $>$-$(N, M)$-cycle-optimality
and $\geq$-$(N', M')$-cycle-optimality are incomparable.  These
observations show that cycle-deal optimality notions form a
(non-linear) hierarchy of efficiency concepts of diverse
strengths. The natural question is whether they can be related to
sequenceability and Pareto-optimality. Obviously, Pareto-optimality
implies both $>$-cycle-optimality and $\geq$-cycle-optimality. \hidden{The
  following proposition, whose proof is inspired by \citet[Proposition
  1,necessity]{BramsKing05}, is much stronger:} An easy adaptation of
the proof of Proposition~\ref{prop:POseq} leads to the following
stronger result:

\begin{prop}\label{pro:cycle-sequence}
  An allocation $\vect{\pi}$ is sequenceable if and only if it
  is $>$-$n$-cycle optimal (with $n = |\mathcal{N}|$).
\end{prop}

\begin{proof}
  Let $\vect{\pi}$ be a non-sequenceable allocation. Then by
  Proposition~\ref{prop:sousAlloc}, there is at least one frustrating
  sub-allocation in ${\vect{\pi}}$. Using the same line of arguments
  as in the proof of Proposition~\ref{prop:POseq} we can build a
  $>$-$k$-cycle. Hence $\vect{\pi}$ is not $>$-cycle
  optimal. Conversely, suppose that $\vect{\pi}$ is a
  $>$-$k$-cycle. Then obviously this cycle yields an allocation that
  Pareto-dominates $\vect{\pi}$.
\end{proof}

\hidden{
\begin{proof} Let ${\vect{\pi}}$ be a non-sequenceable
  allocation. From Proposition~\ref{prop:sousAlloc},\hidden{ in a
  non-sequenceable allocation,} there is at least one frustrating
  sub-allocation in ${\vect{\pi}}$. Let ${\vect{\pi}'}$ be such a sub-allocation\hidden{ (that
  can be ${\vect{\pi}}$ itself)}. We will, from ${\vect{\pi}'}$, build
  another sub-allocation dominating it. Let us choose an arbitrary
  agent $a_1$ involved in ${\vect{\pi}'}$, receiving an object not
  among her top ones in ${\vect{\pi}'}$. Let $o_1$ be a top object of
  $a_1$ in ${\vect{\pi}'}$, and let $a_2$ ($\neq a_1$) be the unique
  agent receiving it in ${\vect{\pi}'}$. Let $o_2$ be a top object of
  $a_2$.  We can notice that $o_2 \neq o_1$ (otherwise \hidden{ $a_2$ would
  obtain one of her top objects and} ${\vect{\pi}'}$ would not be
  frustrating). Let $a_3$ be the unique agent receiving $o_2$ in
  ${\vect{\pi}'}$, and so on.  Using this argument iteratively, we
  form a path starting from $a_1$ and alternating agents and objects,
  in which two successive agents and objects are distinct. Since the
  number of agents and objects is finite, we will eventually encounter
  an agent which has been encountered at a previous step of the
  path. Let $a_i$ be the first such agent and $o_k$ be the last object
  seen before her in the sequence ($a_i$ is the unique agent receiving
  $o_k$). We have built a cycle deal
  $ \langle (a_k, o_{k-1}), (a_{k-1} , o_{k-2}), \cdots, (a_i , o_{k})
  \rangle$ in which all the agents and objects are distinct, and that
  has at least two agents and two objects. This cycle applied to
  ${\vect{\pi}'}$ builds a new sub-allocation by giving to each agent
  in the cycle a top object instead of another less preferred
  object. Applied to ${\vect{\pi}}$, this cycle is a strictly
  improving $n$-cycle deal, and hence, $\vect{\pi}$ is not
  $>$-$n$-cycle optimal.

  Conversely, let $\vect{\pi}$ be an allocation and
  $\langle (i_1, o_{l_1}), \dots, (i_k, o_{l_k}) \rangle$ be a
  $>$-$k$-cycle ($k \leq n$). Then, obviously,
  $\vect{\pi}^{|\{ o_{l_1}, \dots, o_{l_k} \}|}$ is a frustrating
  allocation, because every agent strictly prefers the object given by
  the previous agent in the sequence. By
  Proposition~\ref{prop:sousAlloc}, $\vect{\pi}$ is not sequenceable.
\end{proof}
}

The scale of efficiency introduced in
Section~\ref{sec:Pareto} can then be refined with an intermediate hierarchy
of $>$-cycle optimality notions between sequenceable and
non-sequenceable allocations. As for $\geq$-cycle optimality, it forms a
parallel hierarchy between Pareto-optimal and non-sequenceable
allocations.\footnote{Note that $\geq$-CO and sequenceability are
  incomparable notions. There exist allocations \hidden{--- that we omit for
  space restrictions ---} which are $\geq$-swap optimal but not
  sequenceable and the other way around.}

A corollary of Propositions~\ref{prop:complexity} and
\ref{pro:cycle-sequence} is that checking whether an allocation is
$>$-$n$-cycle optimal can be made in polynomial time (by checking
whether it is sequenceable).

More generally, we can observe that checking whether an allocation is
$(k,k')$-cycle optimal can be done by iterating over all $k$-uples of
agents, and for each one iterating over all possible transfers
involving less than $k'$ objects.\footnote{This is sufficient to also
  run through all the cycles involving strictly less than $k$ agents:
  such a cycle can be simulated just by appending at the end of the
  cycle some agents whose role is just to pass the objects they
  receive to the next agent.} In total, there are $k!\binom{n}{k}$
$k$-uples of agents (which is upper-bounded by $n^{k+1}$). For each
$k$-uple, there are at most
${\left(\sum_{k'' = 0}^{k'}\binom{m}{k''}\right)}^k$ possible
transfers, which is again upper-bounded by $(1 + m)^{kk'}$. Hence, in
total, checking whether an allocation is $(k,k')$-cycle optimal can be
done in time $O(n^{k+1}\times (1+m)^{kk'})$. This is polynomial in $n$
and $m$ if both $k$ and $k'$ are bounded (as for swap deals).


\section{Restricted Domains}
\label{sec:restricted-domains}

We now study the impact of several preference 
restrictions on the hierarchy of efficiency notions introduced in  Section~\ref{sec:deals}. 

\paragraph{Strict preferences on objects}
When the preferences are strict on objects, then obviously every
sequence generates exactly one allocation. The following proposition
is stronger and shows that the reverse is also true:

\begin{prop} \label{prop:StrObj} Preferences are strict on objects iff $s(I)$ is a mapping from $\mathcal{S}(I)$
  to $\mathcal{P}(I)$.
\end{prop}

\begin{proof}  
  If preferences are strict on objects, then each agent has only one
  possible choice at her turn in the sequence of sincere choices \hidden{ line~\ref{line:guess} of Algorithm~\ref{algo:ss}}
  and hence every sequence generates one and only one allocation.

  Conversely, if preferences are not strict on objects, at least one
  agent (suppose w.l.o.g. agent 1) gives the same weight to two
  different objects. Suppose that there is at least $t$ objects ranked
  above. Then obviously, the following sequence
  $ \underbrace{111...111}_{t + 1\text{
      times}}\underbrace{222...222}_{m - t - 1\text{ times}}
  $
  generates two allocations, depending on agent 1's choice at step
  $t+1$.
\end{proof}


\paragraph{Same order preferences}
We say that the agents have
\emph{same order preferences}
\hidden{\citep{Bouveret2015CharacterizingConflicts}} if there is a permutation
$\eta : \mathcal{O} \mapsto \mathcal{O}$ such that for each agent $i$
and each pair of objects $k$ and $l$, if $\eta(k) < \eta(l)$ then
$w(i, \eta(k)) \geq w(i, \eta(l))$.

\begin{prop} \label{prop:poi} All the allocations of an instance with
  same order preferences are sequenceable (and actually cycle-deal optimal). 
  Conversely, if all the
  allocations of an instance are sequenceable, then this instance has
  same order preferences.
\end{prop}

\begin{proof} Let $I$ be an instance with same order preferences, and
  let ${\vect{\pi}}$ be an arbitrary allocation.  In every sub-allocation
  of ${\vect{\pi}}$ at least one agent obtains a top object (because the
  preference order is the same among agents) and hence cannot be
  frustrating. By Proposition~\ref{prop:sousAlloc}, ${\vect{\pi}}$ is
  sequenceable.

  Conversely, let us assume for contradiction that $I$ is an instance
  not having same order preferences. Then there are two distinct
  objects $k$ and $l$ and two distinct agents $i$ and $j$ such that
  $w(i,k) \geq w(j,k)$ and $w(i,l) \leq w(j,l)$, one of the two
  inequalities being strict (assume w.l.o.g. the first one). The
  sub-allocation ${\vect{\pi}}^{|\{k, l\}}$ such that
  $\pi^{|\{k, l\}}_i = \{l\}$ and $\pi^{|\{k, l\}}_j = \{k\}$ is
  frustrating. By Proposition~\ref{prop:sousAlloc}, every allocation
  ${\vect{\pi}}$ containing this frustrating sub-allocation (hence
  such that $l \in \pi_i$ and $k \in \pi_j$) is non-sequenceable.
\end{proof}

Let us now characterize the instances for which $s(I)$ is a one-to-one
correspondence.

\begin{prop} \label{prop:SOP}
  For a given instance, the following two statements are equivalent.
  \begin{itemize}
  \item[(A)] Preferences are strict on objects and in the same order.
  \item[(B)] The relation $s(I)$ is a one-to-one correspondence.
  \end{itemize}
\end{prop}

The proof is a consequence of
Propositions~\ref{prop:StrObj} and \ref{prop:poi}.




\hidden{
\subsubsection{0--1 preferences}
\fixme{Interesting?}
}

\paragraph{Single-peaked preferences}

An interesting domain restriction are single-peaked preferences
\citep{Black1948,ElkindLP16}, which, beyond voting, is also relevant
in resource allocation settings \citep{Bade2017,Damamme15}. Formally,
in this context, single-peakness can be defined as follows.

There exists a linear order $\triangleright$ over the set of objects
$\mathcal{O}$. Let $top(i)$ be the preferred object of $i$.  An agent
$i$ has \emph{single-peaked} preferences wrt. $\triangleright$ if, for
any two objects $(a,b) \in \mathcal{O}$ such that either
$top(i) \triangleright b \triangleright a$ or
$a \triangleright b \triangleright top(i)$ (i.e. lying on the same
``side'' of the agent's peak), it is the case that $i$ prefers $b$
over $a$.


\hidden{
\fixme{It could be useful to define single-peakness
  formally. Does it mean for instance that the preferences are strict
  on objects? Nico: In the basic version, yes. But there are variants / extensions for weak orders, see Lackner, in particular, and Fitzsimmons TARK-2015. Problem: there are several possible definitions then. Worth a look?}
  }

Interestingly, when preferences are single-peaked, the hierarchy of
$N$-cycle optimality collapses at the second level:

\begin{prop}
  \label{pro:single-peak-pref}
  If all the preferences are single-peaked (and additive), then an
  allocation $\vect{\pi}$ is $\geq$-$n$-cycle optimal iff it is swap-optimal.
\end{prop}

\begin{proof}(Revisiting \cite{Damamme15})
First, note that $\geq$-$n$-cycle optimality trivially implies swap-optimality. Let us now show the conserve.

Let us consider for the sake of contradiction an allocation $\vect{\pi}$ that is swap-optimal and such that 
there exists a $\geq$-$k$-cycle $\mu$, with $k \leq n$. Without loss of generality, let us suppose that 
$\mu=\langle\left(1, \{r_1\}\right), \dots, \left(k, \{r_k\}\right) \rangle$. We show by induction on $k$, 
the length of $\mu$, that such a cycle can not exist. \\

\underline{Base case: $k=2$} A 1-cycle of length $k = 2$ is a swap-deal but as $\vect{\pi}$ is swap-optimal,
no improving swap-deal exists in $\vect{\pi}$ hence the contradiction.\\

\underline{Induction step:} Let us assume that for each $k'$ such that $2 \leq  k' \leq k - 1$, no 
$\geq$-$k'$-cycle exists in $\vect{\pi}$ and let us show that no cycle of length $k$ exists.

To exhibit a contradiction we will need to use the following necessary \citep{Ballester_Haeringer_2011}: 
to be single-peaked, a profile $U$ needs to be worst-restricted, i.e. for any triple of resources 
$R = (r_a, r_b, r_c) \in \mathcal{O}^3$ there always exists a resource $r_j \in R$ such that there exists
an agent $i$ with $r_j \notin \argmin_{k\in R}w(i, k)$ \citep{Sen_1996}.

Because $\mu$ is a $\geq$-$k$-cycle, for all agent $i \neq 1$ involved in $\mu$ we have $r_{i-1}\succ_i r_i$
and $r_k\succ_1 r_1$. As no $\geq$-$k'$-cycle exists, with $k' < k$, for all agents $i \neq 1$ involved in 
$\mu$ and for all resources $r$ in $\mu$, $r\neq r_i$ and $r\neq r_{i-1}$, we have $r_i\succ_i r$. Moreover for all 
resource $r$ in $\mu$, $r\neq r_1$ and $r\neq r_k$, we have $r_1 \succ_1 r$. If the preferences do not respect 
these conditions, a $\geq$-$k'$-cycleexists with $k' < k$.

Because the profile is worst-restricted, for all the triple of resources $R$ in $\{r_1, \dots, r_k\}$, at most 
two resources of $R$ can be ranked last among $R$ by the agents. Let us call $r_w$ one of these resources ranked 
last by agent $i_l$ and held by agent $i_w$. Thanks to the previous paragraph, we know that $\best(\mathcal{O}, i_w) = r_{w-1}$ 
and so, because her preferences are single-peaked, $i_w$ puts $r_{w+1}$ in last position among $r_{w-1}, r_w, r_{w+1}$. 
The same holds for agent $i_{w+1}$ who ranks $r_{w-1}$ in last position among $r_{w-1}, r_w, r_{w+1}$ (because $top(i_{w+1}) = r_w$). 
Therefore when we focus only on the three resources $r_{w-1}, r_w, r_{w+1}$, each of them is ranked last among them 
by one agent which violates the condition of worst-restriction. 
The contradiction is set, no $\geq$-$k$-cycle exists in $\overrightarrow{\pi}$.
\end{proof}

Together with Proposition~\ref{pro:cycle-sequence},
Proposition~\ref{pro:single-peak-pref} gives another interpretation
of sequenceability in this domain:

\begin{corollary}
  \label{coro:single-peak-seq}
  If all the preferences are single-peaked (and additive), then an allocation
  $\vect{\pi}$ is sequenceable if and only if it is swap-optimal.
\end{corollary}

Proposition~1 by \cite{Damamme15} is much stronger than our
Corollary~\ref{coro:single-peak-seq}, because it shows that
swap-optimality is actually equivalent to Pareto-efficiency \emph{when
  each agent receives a single resource}.  Unfortunately, in our
context where each agent can receive several items, this is no longer
the case, as the following example shows:

\begin{example} \label{exSP}  
  Consider this instance, single-peaked with respect to
  $1 \triangleright \dots \triangleright 6$:

 \[  \left(
     \begin{array}{cccccc}
       \dag \al 1 & \al 2 & 3 & 4 & 5 & \dag 6\\
       1 & \dag 3 & \al 4 & \al 5 &  \dag 6 & 2\\
       1 & 2 & \dag 4 & \dag 5 & \al 6 & \al 3
     \end{array}
   \right) \]
 
 The circled allocation is swap-optimal, but Pareto-dominated by the
 allocation marked with dags.
\end{example}


\section{Envy-Freeness and CEEI}
\label{sec:CEEI}

\hidden{In the previous section we have investigated the link existing between
efficiency and sequenceability. }The use of sequence of sincere choices
can also be motivated by the search for a \emph{fair} allocation
protocol. 
\hidden{In this section, w} We will focus on two fairness properties,
envy-freeness and competitive equilibrium from equal income and
analyze their link with sequenceability.  Envy-freeness
\citep{Tinbergen53,Foley67,Varian1974} is probably one of the most
prominent fairness properties:

\begin{definition} \label{def:EF} Let $I$ be an add-MARA instance and
  ${\vect{\pi}}$ be an allocation.  ${\vect{\pi}}$ verifies the
  \emph{envy-freeness property} (or is simply \emph{envy-free}), when
  $u_i(\pi_i) \geq u_i(\pi_j)$,
  $\forall (i, j) \in \mathcal{N}^2$ (no agent strictly prefers the
  share of any other agent).
\end{definition}

The notion of \emph{competitive equilibrium} is an old and well-known
concept in economics
\citep{Walras1874Elements,Fisher1891MathematicalInvestigations}. If
equal incomes are imposed among the stakeholders, this concept becomes
the \emph{competitive equilibrium from equal incomes} \citep{Moulin03},
yielding a very strong fairness concept which has been recently
explored in artificial intelligence
\citep{Othman2010,Budish11,Bouveret2015CharacterizingConflicts}.

\begin{definition}\label{def:CEEI}
  Let $I = (\mathcal{N}, \mathcal{O}, w)$ be an add-MARA instance, ${\vect{\pi}}$
  an allocation, and $\vect{p} \in [0, 1]^m$ a vector of
  prices.  A pair $({\vect{\pi}}, \vect{p})$ is said to form a
  \emph{competitive equilibrium from equal incomes (CEEI)} if
  \[
    \forall i \in \mathcal{N} : \pi_i \in \argmax_{\pi \subseteq \mathcal{O}} \left\{
      u_i(\pi) :
      \sum_{k \in \pi} p_k \leq 1 \right\}.
  \]
  In other words, $\pi_i$ is one of the maximal shares that
  $i$ can buy with a budget of 1, given that the price of
  each object $k$ is $p_k$.

  We will say that allocation ${\vect{\pi}}$ satisfies the CEEI test (is a
  CEEI allocation for short) if there exists a vector $\vect{p}$ such
  that $({\vect{\pi}}, \vect{p})$ forms a CEEI.
\end{definition}

As \cite{Bouveret2015CharacterizingConflicts} and
\cite{Branzei2015ComputationalFair} have shown, every CEEI allocation
is envy-free in the model we
use.  
In this section, we investigate the question of whether an envy-free
or CEEI allocation is necessarily sequenceable. For envy-freeness, the
answer is negative.

\begin{prop}
  There exists non-sequenceable envy-free allocations, even if the
  agents' preferences are strict on shares.
\end{prop}

\begin{proof} A counterexample with strict preferences on shares is
  given in Example~\ref{ex:EFnonPOprefsStrictes} above, for which we
  can check that the circled allocation ${\vect{\pi}}$ is envy-free
  and non-sequenceable.
\end{proof}

However, for CEEI, the answer is positive:

\begin{prop} \label{prop:CEEI} Every CEEI allocation is
  sequenceable.
\end{prop}

It was already known that every CEEI allocation is Pareto-optimal if
the preferences are strict on shares
\citep{Bouveret2015CharacterizingConflicts}. From
\hidden{Observation~\ref{rem:str} and} Proposition~\ref{prop:POseq},
if the preferences are strict on shares, then every CEEI allocation is
sequenceable. But Proposition~\ref{prop:CEEI} is more general: no
assumption is made on the strict order on shares (nor on objects).

Note that a CEEI allocation can be ordinally necessary
Pareto-dominated, as the following example shows.
\[
\left(
  \begin{array}{cccc}
    \dag\al 2 & \dag 3 & 3 & \al 2\\
    2 & 3 & \dag\al 4 & 1 \\
    0 & \al 4 & 2 & \dag 4
  \end{array}
\right)
\]
The circled allocation is CEEI (with prices 0.5, 1, 1, 0.5) but is
ordinally necessary (hence also additively) dominated by the
allocation marked with $\dag$.

\hidden{
\begin{proof}(Sketch) Suppose ${\vect{\pi}}$ is non-sequenceable but CEEI. There must exist an allocation ${\vect{\pi}'}$ reachable with an improving cycle involving agents in $\cal C$.  
By summing the equations~of price constraint and optimality for CEEI, provided that all
shares are disjoint, we reach a contradiction:

  \begin{eqnarray*} \label{eq:C}
    p(\bigcup_{j \in \cal C} \pi_j) \leq |\cal C|
    & \text{and} &
                  p(\bigcup_{j \in \cal C} \pi'_j) > |\cal C|
  \end{eqnarray*}  
\end{proof}
}

\begin{proof}  We will show that no allocation can be at the
  same time non-sequenceable and CEEI. Let ${\vect{\pi}}$ be a non-sequenceable
  allocation.  We can use the same terms and notations than in the
  proof of Proposition~\ref{prop:POseq}, especially concerning the
  dominance cycle.
  
  Let $\mathcal{C}$ be the set of agents concerned by the
  cycle. ${\vect{\pi}}$ contains the following shares:
  \[    
    \pi_{a_i} = \{ o_k \} \cup \tau_i \ \ \ \
    \pi_{a_{i+1}} = \{ o_i \} \cup \tau_{i+1} \ \ \  \ .... \ \ \ \
    \pi_{a_k} = \{ o_{k-1} \} \cup \tau_k 
  \]
  whereas the allocation ${\vect{\pi}}'$ that dominates it, contains the
  following shares:
  \[
    \pi'_{a_i} = \{ o_i \} \cup \tau_i \ \ \ \
    \pi'_{a_{i+1}} = \{ o_{i+1} \} \cup \tau_{i+1} \ \ \ \
    .... \ \ \ \
    \pi'_{a_k} = \{ o_{k} \} \cup \tau_k
  \]
  the other shares being unchanged from ${\vect{\pi}}$ to ${\vect{\pi}}'$.
  
  Suppose that ${\vect{\pi}}$ is CEEI. This allocation must satisfy two
  kinds of constraints. First, ${\vect{\pi}}$ must satisfy the price
  constraint. If we write $p(\pi) \Def \sum_{k\in\pi} p_k$, we have,
  $\forall i\in{\cal C}$, $p(\pi_i) \leq 1$ (1).
  
  Next, ${\vect{\pi}}$ must be optimal: every share having a higher utility
  for an agent than her share in ${\vect{\pi}}$ costs strictly more than
  1. Provided that
  $\forall i\in{\cal C} : u_i(\pi'_i) >
  u_i(\pi_i)$
  (because ${\vect{\pi}}'$ substitutes more preferred objects to less
  preferred objects in ${\vect{\pi}}$), this constraint can be written as
  $\forall i\in{\cal C}$, $p(\pi'_i) > 1$ (2).
  
  By summing equations~(1) and (2), provided that all
  shares are disjoint, we obtain
  \begin{eqnarray*} \label{eq:C}
    p(\bigcup_{j \in \cal C} \pi_j) \leq |\cal C|
    & \text{and} &
                  p(\bigcup_{j \in \cal C} \pi'_j) > |\cal C|
  \end{eqnarray*}
  Yet,
  $\bigcup_{j \in \cal C} \pi_j=\bigcup_{i \in \cal
    C}\pi'_j$
  (because the allocation ${\vect{\pi}}'$ is obtained from ${\vect{\pi}}$ by
  simply swapping objects between agents in $\mathcal{C}$). The two
  previous equations are contradictory.
\end{proof}



\section{Experiments}
\label{sec:expe}

We have exhibited in Section~\ref{sec:Pareto} a ``hierarchy of
allocation efficiency'' made of several steps: Pareto-optimal (PO),
sequenceable (Seq), \{cycle-deal-optimal\}, non-sequenceable (--). A
natural question is to know, for a given instance, which proportion of
allocations are located at each level of the scale. We give a first
answer \hidden{in this section} by experimentally studying the
distribution of allocations between the different levels. For
cycle-deal optimality, we focus on the simplest type of deals, namely,
$>$-swap-deals. We thus have a linear scale of efficiency concepts,
from the strongest to the weakest: PO $\to$ Seq $\to$ Swap $\to$
--. We also analyze the relation between efficiency and various
notions of fairness by linking this latter scale with the 6-level
scale of fairness introduced by
\cite{Bouveret2015CharacterizingConflicts}: CEEI $\to$ Envy-Freeness
(EF) $\to$ min-max share (mFS) $\to$ proportionality (PFS) $\to$
max-min share (MFS) $\to$ --.  \hidden{Our experimental protocol is
  the following. } We generate 50  add-MARA
instances involving 3 agents-8 objects, using two different models. For both models, a set of
weights are uniformly drawn in the interval $\llb 0, 100 \rrb$ and the
instances are then normalized. For the second model, these weights are
reordered afterwards to make the preferences single-peaked. For each
instance, we generate all 6561 allocations, and identify for each of
them the \emph{highest} level of fairness and efficiency
satisfied. The average number of allocations with min-max interval is
plotted as a box for each level on a logarithmic scale in
Figure~\ref{fig:graphe01}. The figure also shows for each fairness
criterion the proportion of allocations that satisfy each efficiency
criterion, on a linear scale.

\begin{figure*}[htbp]
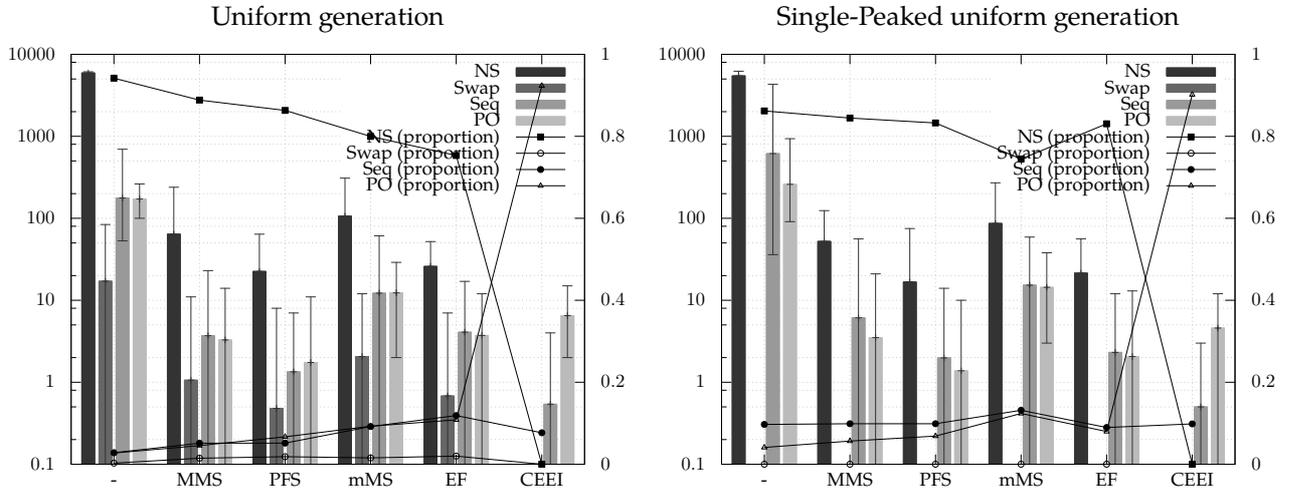

  {\parbox{0.48\textwidth}{{\makebox[1cm]{}\hfill Uniform generation \hfill\hbox{}}\\
      \scalebox{0.7}{
        \input{graphes/graphe01}
      }
  }} \hfill
  \parbox{0.48\textwidth}{{\makebox[1cm]{}\hfill Single-Peaked uniform generation \hfill\hbox{}}\\
    \scalebox{0.7}{
      \input{graphes/graphe02}
    }
  }
  \caption{Distribution of the number of allocations by pair of (efficiency, fairness) criteria.}
  \label{fig:graphe01}
\end{figure*}

Note that some fairness and efficiency tests require to solve \np-hard
or \conp-hard problems (MMS, mMS, and PO tests). These tests are
delegated to an external ILP solver. This is especially interesting
for the CEEI test which is known to be \np-hard
\citep{Branzei2015ComputationalFair}, and for which, to the best of
our knowledge, no practical method had been described before.  The
implementation is available as a fully documented and tested Free
Python library.\footnote{Available at:
  \url{https://gricad-gitlab.univ-grenoble-alpes.fr/bouveres/fairdiv}.}

We note several interesting facts. First, a majority of allocations do
not have any efficiency nor fairness property (first black bar on the
left). Second, the distribution of allocations on the efficiency scale
seems to be related to the fairness criteria: a higher proportion of
swap-optimal or sequenceable allocations are found among envy-free
allocations than among allocations that do not satisfy any fairness
property, and for CEEI allocations, there are even more Pareto-optimal
allocations than just sequenceable ones. Lastly, the absence of
vertical bar for swap-optimality in the experiments concerning
single-peaked preferences confirms the results of
Corollary~\ref{coro:single-peak-seq}: in this context, no allocation
can be swap-optimal but not Sequenceable; hence, all the allocations
that are swap-optimal are contained in the bars concerning
sequenceable or Pareto-optimal allocations. Similarly, the absence of
bars for swap-optimality and -- (non-sequenceable) in both graphs confirms the result of
Proposition~\ref{prop:CEEI}.



\section{Conclusion}
\label{sec:conclusion}

In this paper, we have shown that picking sequences and cycle-deals
can be reinterpreted to form a rich hierarchy of efficiency
concepts. Many interesting questions remain open, such as the
complexity of computing cycle-deals, the precise relation between
Pareto-efficiency and $\geq$-cycle-optimality or the link between
efficiency concepts and social welfare. One could also think of
further extending the efficiency hierarchy by studying restrictions on
possible sequences (\textit{e.g.} alternating) or extending the types
of deals to non-cyclic ones.

\bibliography{partage}


\end{document}